\documentclass[letterpaper, 10pt, conference]{ieeeconf}
\IEEEoverridecommandlockouts  % This command is only needed if you want to use the \thanks command
\usepackage{amsmath}
\usepackage{amssymb}
\usepackage{bm}
\usepackage{color}
\usepackage{graphicx}
\usepackage{hyperref}
\usepackage[misc]{ifsym}

\newcommand{\bfepsilon}{\boldsymbol{\epsilon}}

\newcommand{\bfnu}{\boldsymbol{\nu}}

\newcommand{\bftau}{\boldsymbol{\tau}}

\newcommand{\bfomega}{\boldsymbol{\omega}}

\newcommand{\bfPhi}{\boldsymbol{\Phi}}

\newcommand{\bfPsi}{\boldsymbol{\Psi}}

\newcommand{\bfa}{{\bm{a}}}

\newcommand{\bfd}{{\bm{d}}}

\newcommand{\bff}{{\bm{f}}}
\newcommand{\bfg}{{\bm{g}}}

\newcommand{\bfl}{{\bm{l}}}
\newcommand{\bfm}{{\bm{m}}}
\newcommand{\bfn}{{\bm{n}}}

\newcommand{\bfp}{{\bm{p}}}
\newcommand{\bfq}{{\bm{q}}}
\newcommand{\bfr}{{\bm{r}}}
\newcommand{\bfs}{{\bm{s}}}
\newcommand{\bft}{{\bm{t}}}
\newcommand{\bfu}{{\bm{u}}}
\newcommand{\bfv}{{\bm{v}}}
\newcommand{\bfw}{{\bm{w}}}
\newcommand{\bfx}{{\bm{x}}}

\newcommand{\bfA}{\mathbf{A}}
\newcommand{\bfB}{\mathbf{B}}
\newcommand{\bfC}{\mathbf{C}}

\newcommand{\bfE}{\mathbf{E}}
\newcommand{\bfF}{\mathbf{F}}
\newcommand{\bfG}{\mathbf{G}}

\newcommand{\bfL}{\mathbf{L}}

\newcommand{\bfP}{\mathbf{P}}
\newcommand{\bfQ}{\mathbf{Q}}

\newcommand{\bfU}{\mathbf{U}}

\newcommand{\calC}{{\cal C}}

\newcommand{\calO}{{\cal O}}

\newcommand{\calT}{{\cal T}}

\newcommand{\bbA}{{\mathbb{A}}}

\newcommand{\bbP}{{\mathbb{P}}}

\newcommand{\colvec}[1]{\left[\begin{array}{c} #1 \end{array}\right]}
\newcommand{\defeq}{\stackrel{\mathrm{def}}{=}}
\renewcommand{\th}[1]{\ensuremath {#1}^{\textrm{th}}}
\def\CWC{\mathrm{CWC}}
\def\DSL{\textsf{\textsc{ds-l}}}
\def\DSR{\textsf{\textsc{ds-r}}}

\def\SSL{\textsf{\textsc{ss-l}}}
\def\SSR{\textsf{\textsc{ss-r}}}
\def\all{\text{all}}
\def\bbA{\mathbb{A}}
\def\bbP{\mathbb{P}}
\def\conv{\textit{conv}}
\def\ds{\textsf{\textsc{ds}}}
\def\eg{\emph{e.g.},~}
\def\ie{\emph{i.e.},~}
\def\pdd{\ddot{\bfp}}
\def\pd{\dot{\bfp}}

\def\qd{\dot{\bfq}}
\def\rays{\textit{rays}}
\def\rem{\textrm{rem}}
\def\sigmabfa{\sigma_{\hat\bfa}}
\def\ss{\textsf{\textsc{ss}}}
\def\xdd{\ddot{x}}
\def\xt{\tilde{x}}
\def\yt{\tilde{y}}
\def\ydd{\ddot{y}}
\def\zdd{\ddot{z}}

\newtheorem{proposition}{Proposition}

\title{\LARGE \bf
    Multi-contact Walking Pattern Generation based on 
    \\ Model Preview Control of 3D COM Accelerations
}

\author{St\'ephane Caron$^{1}$ and Abderrahmane Kheddar$^{1,2}$% <-this % stops a space
    \thanks{*This work is supported in part by H2020 EU project COMANOID
    \url{http://www.comanoid.eu/}, RIA No 645097.}% <-this % stops a space
    \thanks{$^{1}$CNRS-UM2 LIRMM, IDH group, UMR5506, Montpellier, France.}%
    \thanks{$^{2}$CNRS-AIST Joint Robotics Laboratory (JRL), UMI3218/RL. \newline
    Corresponding author: {\tt\footnotesize stephane.caron@normalesup.org}}%
}

\begin{document}

\maketitle
\thispagestyle{empty}
\pagestyle{empty}

\begin{abstract}
    We present a multi-contact walking pattern generator based on
    preview-control of the 3D acceleration of the center of mass (COM). A key
    point in the design of our algorithm is the calculation of
    contact-stability constraints. Thanks to a mathematical observation on the
    algebraic nature of the frictional wrench cone, we show that the 3D volume
    of feasible COM accelerations is always an upward-pointing cone. We
    reduce its computation to a convex hull of (dual) 2D points, for which
    optimal $\calO(n \log n)$ algorithms are readily available. This reformulation
    brings a significant speedup compared to previous methods, which allows us
    to compute time-varying contact-stability criteria fast enough for the control
    loop. Next, we propose a conservative \emph{trajectory-wide}
    contact-stability criterion, which can be derived from COM-acceleration
    volumes at marginal cost and directly applied in a model-predictive
    controller. We finally implement this pipeline and exemplify it with the
    HRP-4 humanoid model in multi-contact dynamically walking scenarios.
\end{abstract}

\section{Introduction}

Years ago, humanoid robots were considered as research platforms with vague
perspectives in terms of concrete applications. Without much conviction, they
were envisioned for entertainment, as receptionists, or as a high-tech
show-case for other businesses. Some projects are challenging humanoids to be a
daily companion or an assistant for frail
persons\footnote{\url{http://projetromeo.com/}}. The DARPA robotics challenge
boosted the idea that humanoid robots can operate in disaster interventions.
The challenge exhibited interesting developments while highlighting the road
ahead. Nowadays, Airbus Group seriously envisions humanoids as manufacturing
robots to act in large-scale airliner assembly lines. What makes humanoid
robots a plausible solution in these applications is their physical ability to
move in confined spaces, on non-flat floors, using stairs, etc. In such
environments, there are large parts where the robot has to walk robustly. 

\begin{figure}[!t]
    \centering
    \includegraphics[width=0.98\columnwidth]{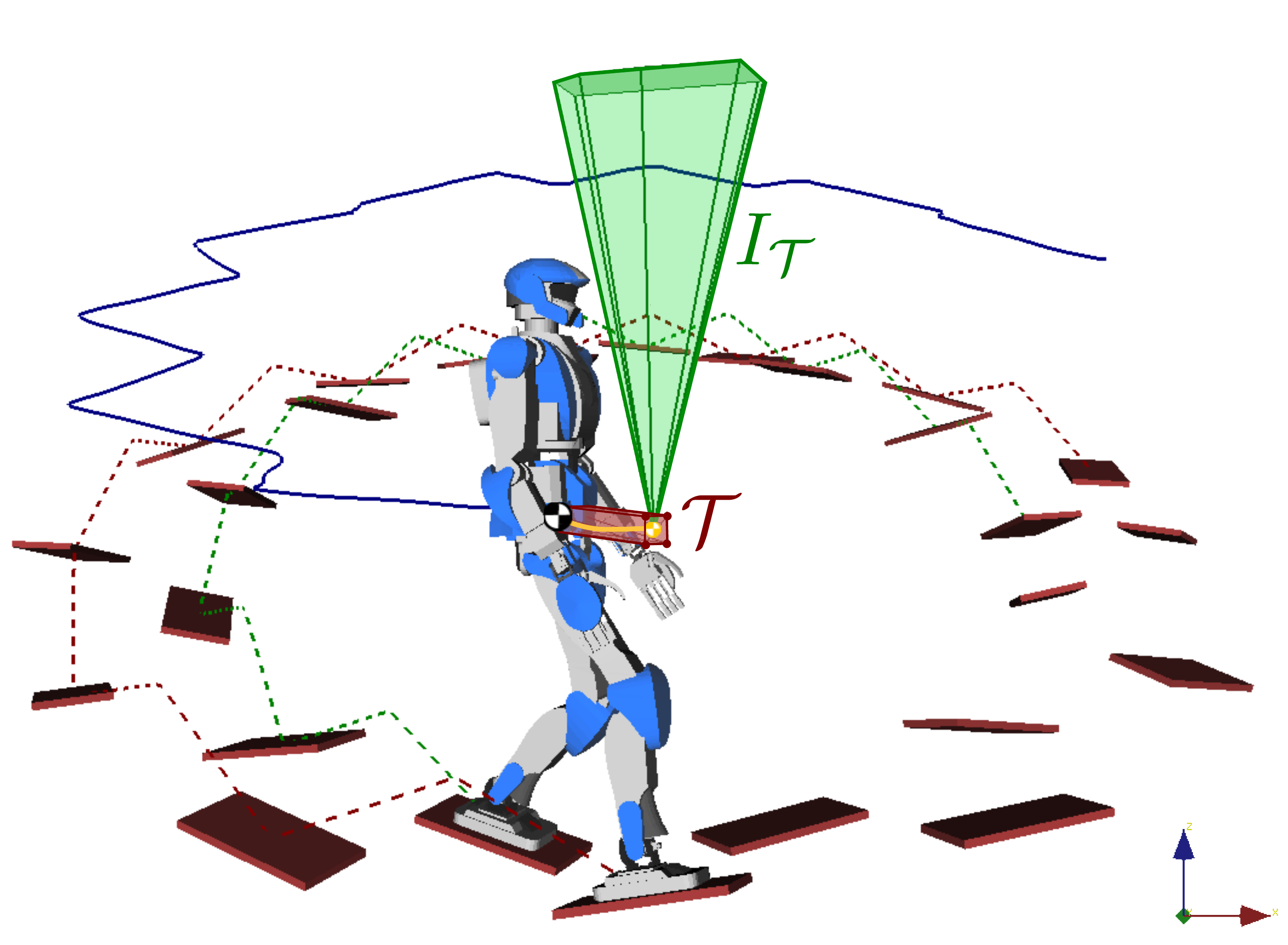}
    \caption{
        HRP-4 walking on a circular staircase with tilted stepping stones using
        a preview controller based on 3D COM accelerations. By bounding future
        COM positions (preview trajectory in yellow) into a polytope $\calT$
        (red box), we derive a \emph{trajectory-wide} contact-stability
        condition, the COM acceleration cone $I_\calT$ (in green), that is
        efficiently computed using a 2D convex hull algorithm. (Scale and
        position of this acceleration cone were chosen arbitrarily for
        depiction purposes.) See the accompanying video~\cite{code} for
        demonstrations of the controller in various multi-contact scenarios.
    }
    \label{fig:staircase}
\end{figure}

Walking robustly on uneven floors is still an open problem in humanoid
research. Recently, Boston Dynamics released an impressive video showing robust
humanoid walks on various
terrains\footnote{\url{https://www.youtube.com/watch?v=rVlhMGQgDkY}}. This
demonstration proves that the goal can be achieved. One key difficulty in
locomotion is that the viability (the ability to avoid falling) of the states
traversed while walking depends on future contacts. This problem can be
addressed geometrically, as done in~\cite{englsberger2015tro} using the
generalized 3D capture point, or using dynamic programming as done
in~\cite{zhao2016rss} walking in the phase-space of the center of mass (COM),
the latter being constrained on predefined surfaces that conform to terrain
shape.

Model Predictive Control (MPC) is another widely applied framework that gives
controllers the required hindsight to tackle this question. Following the
design introduced by Hirukawa et al.~\cite{hirukawa2007icra}, one active line
of research~\cite{audren2014iros, herzog2015humanoids, carpentier2015icra} uses
contact forces as control variables, which produces optimization problems with
simple inequality constraints but a high number of control variables. Another
line of research reduces the problem to the center of mass
motion~\cite{brasseur2015humanoids, naveau2017ral, vanheerden2017ral,
caron2016tro}. Optimization problems are then much smaller, but their
inequality constraints become quadratic (and
non-positive-semidefinite~\cite{vanheerden2017ral}). So far, this problem has only
been addressed for walking on parallel horizontal surfaces:
in~\cite{brasseur2015humanoids}, by bounding vertical COM accelerations to keep
the formulation linear, and in~\cite{naveau2017ral, vanheerden2017ral}, where
Sequential Quadratic Programming was used to cope with quadratic inequalities.

In this paper, we introduce a method that decouples the quadratic inequalities
of the general multi-contact problem into pairs of linear constraints, thus
opening the way for resolution with classical MPC solvers.

\section{Background}

\subsection{Screw algebra}

Humanoid robots are commonly modeled as a set of rigid bodies and joints whose motion can be 
described by \emph{screws}~\cite{featherstone2014}. A screw $\bfs_O = (\bfr, \bfm_O)$ is a
vector field generated by two vectors: its \emph{resultant} $\bfr$ and its
\emph{moment} $\bfm_O$ at a given point $O$. From $\bfm_O$ and $\bfr$, the
moment at any other point $P$ results from the Varignon formula:
\begin{eqnarray}
    \label{varignon}
    \bfm_P & = & \bfm_O + \overrightarrow{PO} \times \bfr.
\end{eqnarray}

The generalized velocity of a rigid body, called \emph{twist}, is the screw
$\bft_O = (\bfv_O, \bfomega)$ with resultant $\bfomega$, the angular velocity, 
and moment $\bfv_O$, its velocity at given point $O$. A
generalized force acting on the body, called \emph{wrench}, is a screw
$\bfw_O = (\bff, \bftau_O)$ with resultant net force $\bff$, 
and net moment $\bftau_O$, at a reference point $O$. 
Although the coordinate vector $\bfs_P$ of a screw depends on the point
$P$ where it is taken, the screw itself does not depend on the choice of $P$ as
a consequence of the Varignon formula~\eqref{varignon}. We denote
screws with hats $(\hat\bft$, $\hat\bfw)$ and their coordinates with
point subscripts $(\bft_O$, $\bfw_O)$.

Twists and wrenches live in two dual spaces~\cite{featherstone2014}: the motion
space $\mathsf{M}^6$ and the force space $\mathsf{F}^6$. The scalar product
between a twist $\hat\bft \in \mathsf{M}^6$ and a wrench $\hat\bfw \in
\mathsf{F}^6$ is given by:
\begin{equation}
    \hat\bft \cdot \hat\bfw \ \defeq \ \bft_O \cdot \bfw_O \ = \ \bfv_O \cdot
    \bff + \bfomega \cdot \bftau_O.
\end{equation}
From~\eqref{varignon}, this number does not depend on the
point $O$ where it is computed. When $\hat\bft$ and $\hat\bfw$ are 
acting on a single rigid body, their product is the
instantaneous power of the motion.

\subsection{Newton-Euler equations}

Let $m$ denote the total mass of the robot and $G$ its center of mass (COM). We
write $\bfp_A$ the coordinate vector of a point $A$ in the inertial frame and
denote by $O$ the origin of this frame (so that $\bfp_O = \bm{0}$).
Suppose that contacts between the robot and its environment are described by
$K$ contact points. (This formulation includes surface contacts; see
\emph{e.g.}~\cite{caron2015icra}.) The Newton-Euler equations of motion of the
whole robot are then given by:
\begin{equation}
    \label{newton-euler}
    \left[ \begin{array}{c} 
            m \pdd_G \\ 
            \dot{\bfL}_G
    \end{array} \right]
    \ = \
    \left[ \begin{array}{c} 
            m \bfg \\ 
            {\bm{0}} 
    \end{array}\right]
    \, + \, 
    \sum_{i=1}^K
    \colvec{\bff_{i} \\ {\overrightarrow{GC}_{i}} \times \bff_{i}}
\end{equation}
where $\bfL_G$ is the angular momentum of the robot around $G$, $\bfg = [0\ 0\
{-g}]^\top$ is the gravity vector defined from the gravity constant $g \approx
9.81\, \textrm{m}\,\textrm{s}^{-1}$ and $\bff_i$ is the force exerted onto the
robot at the $\th{i}$ contact point $C_i$. We say that a contact force $\bff_i$
is \emph{feasible} when it lies in the friction cone $\calC_i$ directed by the
contact normal $\bfn_i$, \ie
\begin{eqnarray}
    \label{fric-cones}
    \| \bff_i - (\bff_i \cdot \bfn_i) \bfn_i \|_2 & \leq & \mu_i (\bff_i \cdot
    \bfn_i)
\end{eqnarray}
where $\mu_i$ is the static friction coefficient. The problem of \emph{contact
stability} (also called \emph{force balance}) is to find whole-body motions for
which~\eqref{newton-euler} admits solutions with feasible contact forces
$\{\bff_i\}$.

The Newton-Euler equations describe the components of motion that are
independent from the actuation power of the robot, and play a critical role in
locomotion. Most of today's trajectory generators~\cite{kajita2003icra,
hirukawa2007icra, mordatch2010tog, brasseur2015humanoids, herzog2015humanoids,
englsberger2015tro, vanheerden2017ral} focus on solving these equations and
rely on whole-body controllers to take actuation limits into account at a later
stage of the motion generation process.

\subsection{Wrench cones}

Equations~\eqref{newton-euler}-\eqref{fric-cones} include a large number of
force variables. Although some walking pattern generators chose to work
directly on this representation~\cite{hirukawa2007icra}, another line of
research~\cite{qiu2011isdhm, caron2015rss} found that these force variables can
be eliminated by propagating their inequality constraints~\eqref{fric-cones}
into inequalities on the target rate of change $(m \pdd_G, \dot{\bfL}_G)$ of
the whole-body momentum.

Define the net \emph{contact wrench} by:
\begin{equation}
    \label{cwc-def}
    \bfw_O \ =\ \colvec{\bff \\ \bftau_O}
    \ \defeq\  \sum_{i=1}^K \colvec{\bff_{i} \\ {\overrightarrow{OC}_{i}}
    \times \bff_{i}}
\end{equation}
In matrix form, $\bfw_O = \bfG_O \bff_\all$ where $\bff_\all$ is the stacked
vector of contact forces and $\bfG_O$ is the \emph{grasp matrix}. This wrench
can be directly computed from whole-body motions, as it only differs from the
whole-body momentum by a constant $\hat\bfw^g$ due to gravity.

Next, one can linearize regular friction cones~$\calC_i$ into polyhedral convex
cones $\widetilde{\calC}_i$, so that~\eqref{fric-cones} becomes in matrix form
(see \emph{e.g.}~\cite{caron2015rss} for details):
\begin{equation}
    \label{lin-fric-cones}
    \bff_i \in \widetilde{\calC}_i\ \Leftrightarrow\ \bfF_i \bff_i \leq \bm0
\end{equation}
This form is known as the \emph{halfspace representation} of a polyhedral cone.
From the Weyl-Minkowski theorem, any polyhedron thus described can be
equivalently written as:
\begin{equation}
    \widetilde{\calC}_i \ = \ \conv(\{\bfv_i\}) + \rays(\{\bfr_j\}),
\end{equation}
where $\conv(\{\bfv_i\}) = \{ \sum_i \alpha_i \bfv_i, \forall i\,\alpha_i \geq 0
\sum_i \alpha_i = 1\}$ is the \emph{convex hull} of a set of vertices, and
similarly $\rays(\{\bfr_j\}) = \{ \sum_i \lambda_i \bfr_i, \forall i\, \lambda_i
\geq 0\}$ denotes positive combinations of a set of rays. This form is known as
the \emph{vertex representation} of a polyhedron.

Using suitable conversions between these two
representations~\cite{qiu2011isdhm, caron2015rss}, one can finally compute the
Contact Wrench Cone (CWC) described in halfspace representation by:
\begin{equation}
    \label{cwc-hrepr}
    \bfA_O \bfw_O \leq \bm0
\end{equation}
By construction, a net contact wrench $\bfw_O$ belongs to the CWC if and only
if there exists a set of contact forces $\{\bff_i\}$ satisfying
both~\eqref{newton-euler}-\eqref{cwc-def} and \eqref{lin-fric-cones}. Hence,
the CWC provides a necessary and sufficient condition for the contact stability
of whole-body motions.

\section{Friction Cones are Dual Twists}

Let us consider a row $\bfa = [\bfa_1^\top \bfa_2^\top]^\top$ of the CWC matrix
$\bfA_O$. It defines an inequality constraint of the form
\begin{equation}
    \label{fc1}
    \bfa_1 \cdot \bff + \bfa_2 \cdot \bftau_O \ \leq \ 0
\end{equation}
Applying the Varignon formula~\eqref{varignon}, the cone for the same wrench
$\bfw_G$ taken at a different point $G$ is subject to:
\begin{equation}
    \bfa_1 \cdot \bff + \bfa_2 \cdot (\bftau_G + \overrightarrow{OG} \times
    \bff) \ \leq \ 0
\end{equation}
Using the invariance of the mixed product under circular shift, we can rewrite
the left-hand side as:
\begin{equation}
    \label{fc2}
    (\bfa_1 + \overrightarrow{GO} \times \bfa_2) \cdot \bff + \bfa_2 \cdot
    \bftau_G \ \leq \ 0
\end{equation}
Let us now defined the \emph{dual twist} $\hat{\bfa} \in \mathsf{M}^6$ by:
\begin{equation}
    \colvec{\bfa_O \\ \bfa} \ \defeq \ \colvec{\bfa_1 \\ \bfa_2}
\end{equation}
Equations~\eqref{fc1} and \eqref{fc2} rewrite to:
\begin{eqnarray}
    \label{fcs1} \bfa_O \cdot \bff + \bfa \cdot \bftau_O & \leq & 0 \\
    \label{fcs2} \bfa_G \cdot \bff + \bfa \cdot \bftau_G & \leq & 0
\end{eqnarray}
where $\bfa_G = \bfa_1 + \overrightarrow{GO} \times \bfa_2$. In concise form:
\begin{equation}
    \hat\bfa \cdot \hat\bfw \ \leq \ 0
\end{equation}
This inequality is independent from $O$ where $\bfA_O$ is computed. 
Therefore, the CWC can be interpreted as a set
of dual twists, the coordinates $\bfA_O$ of which one can compute at a fixed
reference point using known techniques~\cite{qiu2011isdhm, caron2015rss}.

This shift in the way of considering the cone has an important implication:
using the Varignon formula, we can now calculate {analytically} the cone
$\bfA_G$ at a mobile point $G$ using a fixed solution $\bfA_O$ and the vector
coordinates $\bfp_G$. Our following contributions build upon this property.

\section{Contact Stability Areas and Volumes}
\label{areas-and-volumes}

\subsection{Static-equilibrium COM polygon}
\label{sep-poly}

Bretl and Lall~\cite{bretl2008tro} showed how static equilibrium can be
sustained by feasible contact forces if and only if the (horizontal projection
of the) center of mass lies inside a specific polygon, 
henceforth called the \emph{static-equilibrium polygon}. 
%They proposed a recursive expansion
%technique to compute this polygon, which was further extended in
%\cite{delprete2016icra} while \cite{caron2016tro, zhang2016ijhr} proposed an
%alternative based on the double-description method.

In fact, the static-equilibrium polygon is embedded in the CWC.
Suppose that its matrix $\bfA_O$ was computed at a given point $O$, and
let $\hat\bfa$ denote a twist of the CWC corresponding to the
inequality~\eqref{fcs1}. In static equilibrium, the whole-body momentum is
zero, so that the net contact wrench $\bfw_G$ at the center of mass $G$ is
simply opposed to gravity:
\begin{equation}
    \bfw_G \ = \ \colvec{\bff \\ \bftau_G} \ = \ \colvec{-m \bfg \\ \bm0}
\end{equation}
Then, expressing \eqref{fcs1} at $G$,~\eqref{fcs2} yields:
\begin{equation}
    \bfa_G \cdot (-m \bfg) + \bfa \cdot \bm0 \ \leq \ 0
\end{equation}
which also writes, since $m>0$:
\begin{equation}
    -(\bfa_O + \bfa \times \bfp_G) \cdot \bfg \ \leq \ 0
\end{equation}
Expanding this scalar product yields:
\begin{equation}
    \label{fsc2d}
    a_{Oz} - a_y x_G + a_x y_G\ \leq\ 0
\end{equation}
where $\bfa_O = [a_{Ox}\ a_{Oy}\ a_{Oz}]^\top$ and $\bfa = [a_x\ a_y\
a_z]^\top$. The set of inequalities~\eqref{fsc2d} over all twists $\hat\bfa$ of
the CWC provides the half-plane representation of the static-equilibrium
polygon. Note that the static-equilibrium polygon
does not depend on the mass, which was not observed in previous
works~\cite{bretl2008tro, delprete2016icra, caron2016tro, zhang2016ijhr}.

In what follows, we will use the following equivalent formulation. Let us
define the \emph{slackness} of~\eqref{fsc2d} by:
\begin{equation}
    \sigma_{\hat\bfa}(x_G, y_G) \ \defeq \ -a_{Oz} + a_y x_G - a_x y_G
\end{equation}
it is the signed distance between $(x_G, y_G)$ and
the supporting line $-a_y x + a_x y + a_{Oz} = 0$ of the corresponding 
static-equilibrium polygon's edge. A point $(x_G, y_G)$ is then inside the polygon
if and only if $\sigmabfa(x_G, y_G) \geq 0$ for all the CWC twists $\hat\bfa$.

\subsection{Vertex enumeration for polygons}
\label{vertex-enum}

The half-plane representation~\eqref{fsc2d} is best-suited for COM feasibility
tests. Meanwhile, the vertex representation is best-suited for planning.
Converting from halfspace to vertex representation is known as the \emph{vertex
enumeration} problem, for which the \emph{double description
method}~\cite{fukuda1996double} has been applied in previous
works~\cite{bouyarmane2009icra, qiu2011isdhm, escande2013ras, caron2015rss}.

For general $d$-dimensional polyhedra, vertex enumeration has polynomial, yet
super-linear time complexity. For example, the Avis-Fukuda
algorithm~\cite{avis1992dcg} runs in $\calO(d h v)$, with $h$ and $v$ the
numbers of hyperplanes and vertices, while the original double-description
method by Motzkin has a worst-case time complexity of $\calO(h^2
v^3)$~\cite{fukuda1996double}. Yet, for $d=2$, the problem boils into computing
the \emph{convex hull} of a set of points, for which optimal algorithms (for
instance \cite{kirkpatrick1986siam}) are known that match the theoretical
lower-bound of $\Omega(h \log v)$.

Our formulation~\eqref{fsc2d} allows us to enumerate vertices in 2D. Let us
assume for now that the origin $(x_G, y_G) = (0, 0)$ lies in the interior of
the polygon, and divide each inequality~\eqref{fsc2d} by $a_{Oz}$ to put the
overall inequality system in polar form:
\begin{equation}
    \label{polar}
    \bfB \colvec{x_G \\ y_G} \ \leq \ \bm1
\end{equation}
We run a convex hull algorithm on the rows of the matrix $\bfB$. By duality,
the cyclic order of extreme points thus computed corresponds to a cyclic order
of adjacent edges for the primal problem. Intersecting pairs of adjacent lines
in this order yields the vertices of the initial polygon. The conversion of
inequalities~\eqref{fsc2d} to~\eqref{polar} being $\calO(n)$, computing the
output polygon is done overall in $\calO(h \log v)$. See the Appendix for a
comparison with existing approaches.

To construct~\eqref{polar}, we assumed that the origin lies inside the polygon.
When this is not the case, one can simply compute the Chebyshev center $(x_C,
y_C)$ by solving a single Linear Program (LP) as detailed \textit{e.g.} in
\cite{boyd2004convex} p. 148. From there, a translation $(x_G', y_G') = (x_G -
x_C, y_G - y_C)$ brings the origin inside the polygon.

\subsection{Pendular ZMP support areas}

Let us revisit the derivation of the pendular ZMP support
area~\cite{caron2016tro} using our new approach. To achieve linear-pendulum
mode of the Newton-Euler equations of the system, the following four equality
constraints are applied to the contact wrench:
\begin{eqnarray}
    \label{lp1} \bfn \cdot \bff & = & m (\bfn \cdot \bfg) \\
    \label{lp2} \bftau_G & = & \bm0
\end{eqnarray}
where $\bfn$ denotes the unit vector normal to the plane in which the ZMP is
taken. In what follows, we suppose that $\bfn$ is opposite to gravity, so that
$\bfn \cdot \bfg = -g$. Equation~\eqref{lp1} is used to linearize the pendulum
dynamics, the ZMP being defined in general by the non-linear formula $\bfp_Z
\defeq \frac{\bfn \times \bftau_O}{\bfn \cdot \bff} + \bfp_O$.
Under Equations~\eqref{lp1}-\eqref{lp2}, the resultant force can be computed
from COM and ZMP positions by~\cite{caron2016tro}:
\begin{equation}
    \label{forcelin}
    \bff \ = \ \frac{mg}{h} \colvec{x_Z - x_G \\ y_Z - y_G \\ 1}
\end{equation}
where $h = z_Z - z_G$ is the constant difference between ZMP and COM altitudes.
This value can be positive or negative: for the sake of exposition, we will
take $h > 0$. Injecting Equations~\eqref{lp1}-\eqref{lp2} into an inequality
constraint~\eqref{fcs2} of the CWC yields:
\begin{equation}
    \label{fgba}
    (\bfa_O + \bfa \times \bfp_G) \cdot \bff + \bfa \cdot \bm0 \ \leq \ 0
\end{equation}
Substituting~\eqref{forcelin} into~\eqref{fgba} yields:
\begin{equation}
    a_i (x_Z - x_G) + b_i (y_Z - y_G) 
    \ \leq \ h \sigma_{\hat\bfa}(\bfp_G)
    \label{link-sep-0}
\end{equation}
where
\begin{equation}
    \colvec{a_i \\ b_i} \ = \ 
    \colvec{a_{Ox} \\ a_{Oy}} + \colvec{
            a_y z_G - a_z y_G \\
    -a_x z_G + a_z x_G}
\end{equation}
Assuming that the COM lies inside the static-equilibrium
polygon,\footnote{Otherwise, center the polygon on its Chebyshev center as
previously.} the right-hand side of this expression is positive
from~\eqref{fsc2d}. The inequality is then expressed in polar form as
\begin{equation}
    \label{zmp-polar}
    \bfB_{\mathrm{ZMP}}(\bfp_G) \colvec{x_Z - x_G \\ y_Z - y_G} \ \leq \ \bm1
\end{equation}
where the origin $(x_Z, y_Z) = (x_G, y_G)$ lies inside the polygon by
construction. As in~\eqref{polar}, a convex hull algorithm can finally be
applied to compute the vertices of the pendular ZMP support area.

\subsection{3D Volume of COM accelerations}
\label{3dvol}

Equation~\eqref{lp1} is a limitation of the linear-pendulum mode in that the
COM trajectory needs to lie in a plane pre-defined by the vector $\bfn$. This
limitation is all the less grounded that, from Equation~\eqref{forcelin},
controlling the ZMP in this mode is equivalent to controlling the resultant
contact force $\bff$. We therefore propose to directly control this force, or
equivalently, to directly control the three-dimensional COM acceleration
$\pdd_G = \frac1m \bff + \bfg$.

Substituting this acceleration into~\eqref{fgba}, one gets:
\begin{equation}
    \label{fcs3}
    (\bfa_O + \bfa \times \bfp_G) \cdot \pdd_G 
    \ \leq \ (\bfa_O + \bfa \times \bfp_G) \cdot \bfg
    %\ = \ g \sigma_{\hat\bfa}(\bfp_G)
\end{equation}
Expanding scalar products, this inequality rewrites to:
\begin{equation}
    \label{link-with-sep}
    a_i \xdd_G + b_i \ydd_G - \sigma_{\hat\bfa} \zdd_G \ \leq \ g
    \sigma_{\hat\bfa}
\end{equation}
(We dropped the argument $\bfp_G$ of $\sigma_{\hat\bfa}$ to alleviate
notations.) Assuming that the COM lies in the interior of the polygon$^3$
($\sigma_{\hat\bfa} > 0$) and that $\zdd_G > -g$, 
\begin{equation}
    \label{force-polar}
    \left(\frac{a_i}{\sigmabfa}\right) \cdot \frac{\xdd_G}{g + \zdd_G} +
    \left(\frac{b_i}{\sigmabfa}\right) \cdot \frac{\ydd_G}{g + \zdd_G} \ \leq \
    1
\end{equation}
This expression is in polar form $\bfB_{\mathrm{3D}}(\bfp_G) [\xt\ \yt]^\top
\leq \bm1$ for the new coordinates:
\begin{equation}
    \colvec{\xt \\ \yt} \ = \ \frac{1}{g + \zdd_G} \colvec{\xdd_G \\ \ydd_G}
\end{equation}
We can enumerate the vertices $\{(\xt_i, \yt_i)\}$ of the corresponding polygon
using a convex hull again. For a given vertical acceleration $\zdd_{G,i} > -g$,
the COM acceleration coordinates $(\xdd_{G,i}, \ydd_{G,i})$ corresponding to a
vertex $(\xt_i, \yt_i)$ are:
\begin{equation}
    \colvec{\xdd_{G,i} \\ \ydd_{G,i}} \ = \ (g + \zdd_{G,i}) \colvec{\xt_i \\
    \yt_i}
\end{equation}
We recognize the equation of a 3D polyhedral convex cone, pointing upward,
with apex located at $(\xdd_G, \ydd_G, \zdd_G) = (0, 0, -g)$ and rays defined by
$\bfr_i = [\xt_i\ \yt_i\ 1]^\top$. Figure~\ref{fig:config3} shows the cone
in a sample contact configuration.

\begin{figure}[t]
    \centering
    \includegraphics[width=0.98\columnwidth]{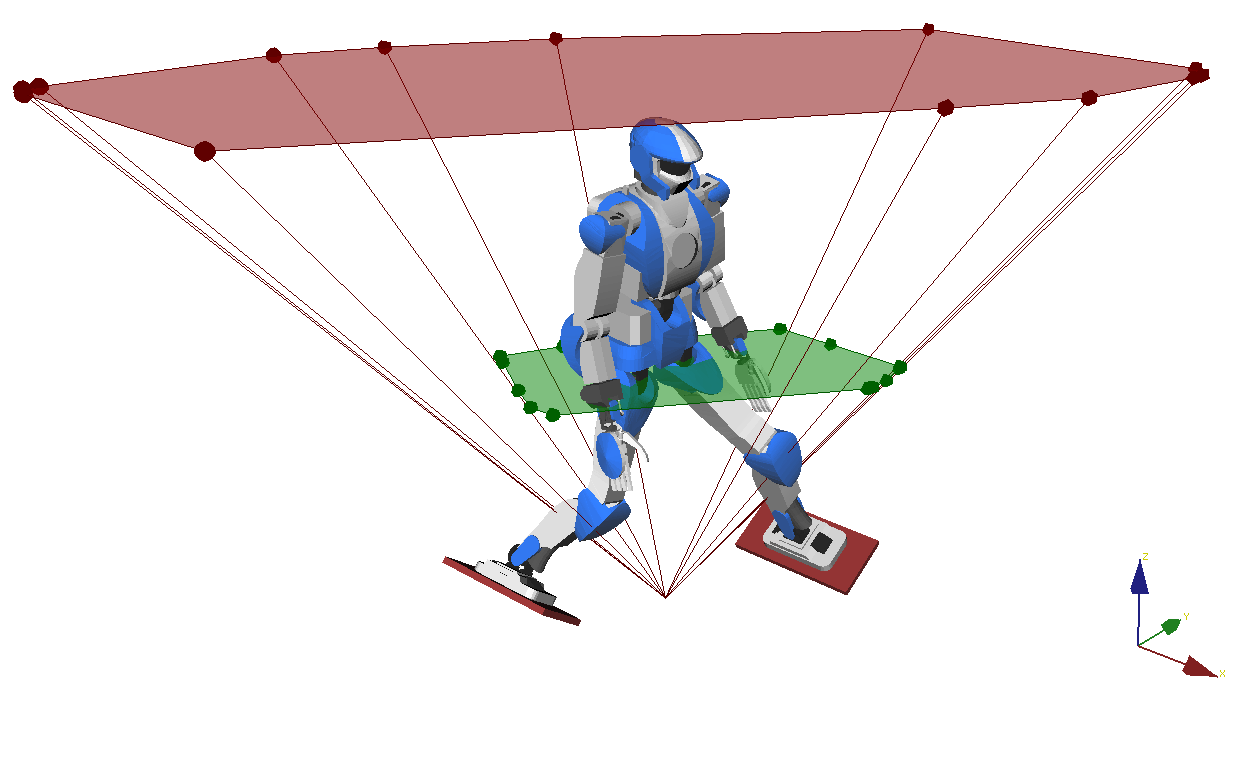}
    \caption{
        Static-equilibrium polygon (in green at the altitude of the center of
        mass) and cone of 3D COM accelerations (in red, with the zero
        acceleration centered on the COM) in a double-support configuration.
        For the latter, the scaling from accelerations to positions is $0.08\
        \textrm{s}^2$, and the cone was cut at $\zdd_G = g$ to show its
        cross-section.
    }
    \label{fig:config3}
\end{figure}

Overall, we have thus both (1) a geometric characterization and (2) an
algorithm to compute the cone of feasible COM accelerations when the angular
momentum is regulated to zero.\footnote{The same derivation can be applied
with non-zero angular-momentum references; however, the question of finding
such references is still open.} This construction generalizes the pendular
support area defined in~\cite{caron2016tro}. Furthermore, as a consequence of
Equations~\eqref{link-sep-0}-\eqref{link-with-sep}, we have the following
property:

\begin{proposition}
    The COM is in the interior of the static-equilibrium polygon if and only if
    the set of feasible COM accelerations (equivalently, whole-body ZMPs) under
    zero angular momentum contains a neighborhood of the origin.
\end{proposition}

In other words, the static-equilibrium polygon is not only related to static
equilibrium: it is also the set of positions from which the robot can
accelerate its center of mass in any direction (with zero angular momentum).
When the COM reaches the edge of this polygon, the zero acceleration touches a
facet of the acceleration cone. Denoting by $\bfd$ the facet normal, all
feasible accelerations $\pdd_G$ are then such that $\bfd \cdot \pdd_G \geq 0$,
\ie $\bfd$ is an ``irresistible'' direction of motion.

\section{Preview control of COM accelerations}
\label{mpc101}

Let the control variable $\bfu$ be the COM acceleration $\bfu :=
\pdd_G$. The discretized COM dynamics with sampling $\Delta T$ are:
\begin{eqnarray}
    \label{disdyn}
    \bfx(k + 1) & = & \bfA \bfx(k) + \bfB \bfu(k)
\end{eqnarray}
where $\bfx(k) = [\bfp_G(k\,\Delta T)^\top \ \pd_G(k\,\Delta T)^\top]^\top$ and, denoting by
$\bfE_3$ the $3 \times 3$ identity matrix,
\begin{equation}
    \bfA \: = \: \left[
        \begin{array}{rr}
            \bfE_3 & \Delta T \bfE_3 \\
            \bm{0}_{3 \times 3} & \bfE_3
    \end{array} \right]
    \quad
    \bfB \: = \: \left[
        \begin{array}{r}
        \frac12 \Delta T^2 \bfE_3 \\
        \Delta T \bfE_3
    \end{array} \right]
\end{equation}
At each control step, a preview controller receives
the current state $\bfx_0 = (\bfp_0, \pd_0)$ and
computes a sequence of controls $\bfu(0), \ldots, \bfu(N)$ driving the system
from $\bfx_0$ to $\bfx(N)$ at the end of the time horizon $T = N \Delta T$ of
the preview window.
By recursively applying~\eqref{disdyn}, $\bfx(k)$ can be
written as a function of $\bfx_0$ and of $\bfu(0), \ldots,
\bfu(k-1)$ (see \eg \cite{audren2014iros}):
\begin{eqnarray}
    \label{iter-disdyn}
    \bfx(k) & = & \bfPhi_k \bfx_0 + \bfPsi_k \bfU(k-1) \\
    \bfU(k) & = & [\bfu(0)^\top \ \cdots \ \bfu(k)^\top]^\top
\end{eqnarray}
A necessary condition for contact-stability throughout the trajectory is that
all accelerations $\bfu(k)$ lie in the COM-acceleration cone
$\calC(\bfp_G(k \Delta T))$. Expanding its inequalities~\eqref{fcs3} for an
arbitrary twist $\hat\bfa \in \CWC$ yields:
\begin{equation}
    \label{csfun1}
    \bfp_G(k\,\Delta T)^\top [-\bfa \times] (\bfu(k) - \bfg) + \bfa_O^\top (\bfu(k) - \bfg) \leq 0
\end{equation}
Let $L$ denote the number of twists in the CWC. By stacking up the $N$
inequalities~\eqref{csfun1}, we get in more concise form:
\begin{equation}
    \label{tensor-form}
    \bfp_G(k\,\Delta T)^\top \bbA_\times (\bfu(k) - \bfg) + \bfA'_O (\bfu(k) - \bfg) \leq \bm0
\end{equation}
where $\bbA_\times$ is a $3 \times L \times 3$ tensor, and $\bfA'_O$
consists of the first three columns of $\bfA_O$. Combining~\eqref{iter-disdyn} 
and~\eqref{tensor-form}, this
condition yields a set of quadratic inequality constraints of the form:
\begin{equation}
    \label{quad-ineq}
    \forall k < N,\ \bfU(k)^\top \bbP_k \bfU(k) + \bfQ_k \bfU(k) + \bfl_k \leq
    \bm0
\end{equation}
where $\bbP_k$ is a $3(k+1) \times L \times 3(k+1)$ tensor, $\bfQ_k$ is a $L
\times 3(k+1)$
matrix $\bfl_k$ is an $L$-dimensional vector. One can thus formulate the
preview control problem as a Quadratically Constrained Quadratic Program
(QCQP). Although a QCQP formulation was successfully applied for walking on
even terrains with variable COM height~\cite{vanheerden2017ral}, we chose not
to do so for the following reasons:
\begin{itemize}
    \item QCQP is a harder class of problems than Quadratic Programming (QP),
        especially when inequality constraints are not
        positive-semidefinite~\cite{vanheerden2017ral} so that the problem
        non-convex. Real-timeness implies that only a small number of SQP
        iterations can be run in the control loop (two
        in~\cite{vanheerden2017ral}), thus with no convergence guarantee.
    \item Constraints \eqref{quad-ineq} are given without any redundancy
        elimination. In practice, eliminating redundancy (in our case, by
        applying a convex hull algorithm) significantly reduces the number of
        inequality constraints (Table~\ref{table:dimrec}).
\end{itemize}
Instead, we propose a trajectory-wide contact-stability criterion that yields
linear inequality constraints, and for which we can apply the convex-hull
reduction from Section~\ref{areas-and-volumes}.

\subsection{Robust trajectory-wide contact-stability criterion}
\label{sec:tubes}

We want to drive the COM from $\bfx_0 = (\bfp_0, \pd_0)$, its current state, to
a goal position $\bfx_T$ through a trajectory $t \in [0, T] \mapsto \bfp_G(t)$.
Due to the initial velocity $\pd_G(0)$ and real-world uncertainties, the
trajectory will not be exactly a line segment $[\bfp_0, \bfp_T]$. Yet, we
assume that it lies within a polyhedral ``tube'' $\calT = \conv(\{\bfnu_1,
\ldots, \bfnu_q\})$ containing $[\bfp_0, \bfp_T]$. A point $\bfp_G \in \calT$
can be written as a convex combination $\bfp_G = \sum_{i=1}^q \alpha_i
\bfnu_i$, where the $\alpha_i$'s are positive and sum-up to one, so that the
inequality~\eqref{tensor-form} becomes:
\begin{equation}
    \label{csfun2}
    \sum_{i=1}^q \alpha_i (\bfnu_i^\top \bbA_\times + \bfA_O') (\pdd_G - \bfg) \ \leq \ \bm0
\end{equation}
A particular way to enforce the negativity of a convex combination is to ensure
that all of its terms are negative. Thus, a sufficient condition for the
satisfaction of~\eqref{csfun2} is
\begin{equation}
    \label{csint1}
    \bfC_\calT (\pdd_G - \bfg) \leq \bm0, \quad
    \bfC_\calT := \colvec{\bfnu_1^\top \bbA_\times + \bfA_O' \\ \vdots \\
    \bfnu_q^\top \bbA_\times + \bfA_O'}
\end{equation}
which is the halfspace representation of the intersection 
\begin{equation}
    I_\calT \ := \ \calC(\bfnu_1) \cap \cdots \cap \calC(\bfnu_q) 
    \ \subseteq \ \calC(\bfp_G)
\end{equation}

\begin{proposition}
    $I_\calT$ is the set of COM accelerations that are feasible everywhere in $\calT$, \ie
    \begin{equation}
        \label{I-bigcap}
        I_\calT = \bigcap_{\bfp \in \calT} \calC(\bfp)
    \end{equation}
\end{proposition}

\begin{proof}
    We showed that $\pdd_G \in I_\calT$ is feasible at any
    position $\bfp_G \in \calT$ (by construction, \eqref{csint1} $\Rightarrow$
    \eqref{csfun2} $\Rightarrow$ \eqref{csfun1}). Conversely, suppose that
    $\pdd_G$ is feasible for all $\bfp_G \in \calT = \conv(\{\bfnu_1, \ldots,
    \bfnu_k\})$. In particular, \eqref{csfun2} holds when $\bfp_G$ is equal to
    any vertex $\bfnu_j$. Thus, $\pdd_G \in \calC(\bfnu_j)$, and since the
    index $j$ can be chosen arbitrarily, $\pdd_G \in \cap_j \calC(\bfnu_j) =
    I_\calT$.
\end{proof}

A trajectory $\bfp_G(t)$ is \emph{contact-stable within $\calT$} if it satisfies:
\begin{equation}
    \forall t \in [0, T],\ \bfp_G(t) \in \calT\ \textrm{and}\ \pdd_G(t) \in I_\calT
\end{equation}

This condition can be compared with the previous trajectory-wide stability
criterion from~\cite{caron2015rss}, where trajectories were computed such that
\begin{equation}
    \forall t \in [0, T],\ \pdd_G(t) \in \calC(\bfp_G(t))
\end{equation}
The implicit assumption behind such a strategy is that either
trajectory-tracking is perfect, or small deviations $\delta \bfp_G$ from the
reference $\bfp_G(t)$ can be coped with if $\pdd_G$ lies ``inside enough'' of
its cone. By taking $\calT$ around the trajectory, we explicitly model how far
$\pdd_G$ needs to be inside the cone to cope with a set of deviations $\delta
\bfp_G$. In this sense, our criterion is a \emph{robust} condition for
trajectory-wide contact stability, with explicit modelling of the robustness
margin.

Although the tube-wise intersection formalized by Equation~\eqref{I-bigcap}
seems like a severely restrictive condition, we noticed (to our surprise) that
the acceleration cones left after intersection are still sufficient for
locomotion in challenging scenarios (see \eg Section~\ref{sec:xp}).

\begin{table}[th]
    \centering
    \caption{
        Effect of the dual convex-hull reduction on the size of the
        inequality-constraint matrix $\bfC_\calT$. Averages are given on $26$
        matrices (one for each step of the circular staircase).
    }
    \label{table:dimrec}
    \begin{tabular}{rcc}
        & $|\bfC_\calT|$ & $|\bfC'_\calT|$ \\
        \hline 
        Single support & $376 \pm 226$ & $6 \pm 2$ \\
        Double support & $484 \pm 221$ & $6 \pm 2$ \\
        \hline
    \end{tabular}
\end{table}

\subsection{Integration with preview control}

At each new control step, the current state $\bfx_0 = (\bfp_0, \pd_0)$ of the
center of mass feeds the preview controller. We define the target state at the
end of the time horizon by $\bfx_T = (\bfp_T, \bfv_T)$, where $\bfp_T$ is
determined from the current stance in the step sequence, $\bfv_T$ is a
reference velocity of $0.4$ m.s$^{-1}$ oriented in the direction of motion, and
$T$ is calculated from the time remaining in the current gait phase. From
there, our method goes as follows:
\begin{enumerate}
    \item Define $\calT$ containing $[\bfp_0, \bfp_T]$
        and compute its halfspace representation $\bfP_\calT$ using
        the double-description. In practice, we defined $\calT$ by a polyhedral
        cylinder centered on $[\bfp_0, \bfp_T]$ with square cross-sections and
        a robustness radius of 5 cm. 
    \item Compute the halfspace representation $\bfC_\calT$ of $I_\calT$ from
        the contact wrench cone $\bfA_O$ (which is computed only once per
        support phase).
    \item Reduce $\bfC_\calT$ to polar form $\bfB_\calT {[x\ y]^\top} \leq
        \bm1$ and compute its vertex representation $\bfg + \rays(\{\bfr_i\})$
        using a convex hull algorithm\footnote{We used
        \emph{Qhull}~\cite{barber1996quickhull}, available from
    \url{http://www.qhull.org/}} (Section~\ref{3dvol}).
    \item Compute its non-redundant halfspace-representation $\bfC'_\calT$
        using again the double-description.
\end{enumerate}
Next, we formulate the preview control problem as a quadratic program:
\begin{eqnarray}
    \label{qp-obj}
    \min_{\bfU} & : & \| \bfx(N) - \bfx_T \|^2 + \epsilon \|\bfU\|^2 \\
    \textrm{s.t.} & : & \forall k,\ \bfC'_\calT \bfu(k) \leq \bfC'_\calT \bfg \label{qp-lin-ineq} \\
    & & \forall k,\ \bfP_\calT \bfx(k) \leq \bm1
    \label{qp-poly-ineq}
\end{eqnarray}
where $\bfU \defeq \bfU(N-1)$ is the stacked vector of controls from which all
$\bfx(k)$ and $\bfu(k)$ derive by Equation~\eqref{iter-disdyn}. The inequality
constraints \eqref{qp-lin-ineq} and \eqref{qp-poly-ineq} are a linear
decoupling of \eqref{quad-ineq} via polyhedral bounds $\calT$.

As a matter of fact, using polyhedral bounds can be thought of as a general
linearization technique. In~\cite{brasseur2015humanoids}, a linear decoupling
was also obtained by bounding vertical COM accelerations. Similarly, polytopes
of robust COM positions $\bfp_G$ were obtained in~\cite{caron2015rss,
delprete2016icra} by defining polyhedral bounds on disturbances $\bfepsilon$, thus
eliminating the bilinear coupling between $\bfp_G$ and $\bfepsilon$. 

Eliminating redundancy in the pipeline above is a significant computational
step. From one of our experiments (Table~\ref{table:dimrec}), the number of
lines $|\bfC'_\calT|$ in $\bfC'_\calT$ is two orders of magnitude smaller than
that of $\bfC_\calT$. Given that the number of inequalities~\eqref{qp-lin-ineq}
in the above QP is $N |\bfC'_\calT|$, this makes the difference in practice
between solving a problem of size $100$ versus $10,000$ (we use $N=10$ steps).

\begin{table}[th]
    \caption{
        Breakdown of computation times inside the predictive controller,
        averaged over 2000 calls in Experiment VI.B.
    }
    \label{table:times-stair}
    \centering
    \begin{tabular}{rcr}
        Function & Output & Time (ms) \\
        \hline
        Double description of $\calT$ (\ss~and \ds) & $\bfP_\calT$ & $0.3 \pm 0.1$ \\
        H-representation of $I_\calT$ (\ss~and \ds) & $\bfC_\calT$ & $0.2 \pm 0.1$ \\
        Convex hull of $I_\calT$ (\ss~and \ds) & $\bfC'_\calT$ & $2.4 \pm 1.2$ \\
        Other matrix operations & -- & $0.4 \pm 0.4$ \\
        Solving final QP $(N=10)$& $\bfU$ & $0.2 \pm 0.1$ \\
        \hline
        Total &  & $3.5 \pm 2.1$
    \end{tabular}
\end{table}

\subsection{Locomotion state machine}

Our preview controller is applied to the HRP-4 humanoid in various
environments. The inputs to the controller are the current COM position
$\bfp_0$ and velocity $\pd_0$, the preview time horizon $T$, a target COM
position $\bfp_T$, $\calT$ and its acceleration cone $I_\calT$.
These computation of these inputs is supervised by a Finite State Machine (FSM)
that cycles between four phases $\varphi \in \{\SSL, \DSR, \SSR, \DSL\}$, where
$\textsf{\textsc{ss}}$ (resp. \textsf{\textsc{ds}}) stands for single-support
(resp. double-support), while \textsf{\textsc{l}} and \textsf{\textsc{r}}
indicate that the phase ends on the left and right foot, respectively. Each
phase is thus associated with a unique foot contact. The target COM position
$\bfp_G^*(\varphi)$ of a phase is then taken $0.8$~m above this contact.

Phase durations are set to $T_\ss = 1$~s for single-support and
$T_\ds = 0.5$~s for double-support. At each iteration of the control
loop, the input to the preview controller is decided based on the time
$T_\rem$ remaining until the next phase transition. Let us denote by
$\varphi$ the current phase in the FSM and $\varphi'$ the phase after
$\varphi$. We define preview targets by:
\begin{enumerate}
    \item if $\varphi$ is single-support and $T_\rem < \frac12
        T_\ss$:
        \begin{itemize}
            \item $T \leftarrow T_\rem + T_\ds + \frac12 T_\ss$
            \item $\bfp_G \leftarrow \bfp_G^*(\varphi'')$
        \end{itemize}
    \item otherwise, if $\varphi$ is double-support:
        \begin{itemize}
            \item $T \leftarrow T_\rem + \frac12 T_\ss$
            \item $\bfp_G \leftarrow \bfp_G^*(\varphi)$
        \end{itemize}
    \item otherwise ($\varphi$ is single-support and $T_\rem > \frac12 T_\ss$):
        \begin{itemize}
            \item $T \leftarrow T_\rem$
            \item $\bfp_G \leftarrow \bfp_G^*(\varphi)$
        \end{itemize}
\end{enumerate}
Case 1) switches the target of the preview controller to the next staircase
step in the middle of single-support phases\footnote{The same behavior is
present in~\cite{kajita2003icra}: if the control from Figure~5 of this paper
was followed to the end, the COM velocity would go to zero between each step,
which is not the behavior observed in Figures~7 and 8.}, which forces the robot
to start its next step while allowing it to re-use the kinetic momentum in the
direction of motion.

Note that cases 1) and 2) imply contact switches in the middle of preview
trajectories, which different cones $\bfC'_\calT$ depending on the step $k$ in
the preview problem. To take this into account, we compute the switching step
$k_\rem = T_\rem / \Delta T$ along with two tubes $\calT_\ss \subset \calT_\ds$
and their dual cones (computations between these two overlapping tubes can be
factored; see~\cite{code} for details). The corresponding matrices
$(\bfC'_{\calT_\ss}, \bfP_{\calT_\ss})$ and $(\bfC'_{\calT_\ds},
\bfP_{\calT_\ds})$ are then respectively used in
Equations~\eqref{qp-lin-ineq}-\eqref{qp-poly-ineq} for $k \leq k_\rem$ and $k >
k_\rem$.

A breakdown of computation times inside the overall predictive controller is
reported in Table~\ref{table:times-stair}. All computations were run on an
average laptop computer (Intel\textsc{(r)} Core\textsc{(tm)} i7-6500U CPU @
2.50 Ghz).

\subsection{Whole-body controller}

The last step of the pipeline is to convert task objectives, such as COM or
foot positions, into joint commands sent to motor controllers. For this, we
used our own solver implemented in the \emph{pymanoid}
library.\footnote{\url{https://github.com/stephane-caron/pymanoid}} It solves a
single quadratic program on five weighted tasks (see~\cite{caron2016tro} for
details), by decreasing priority: support foot, swing foot, COM tracking,
constant angular-momentum, and posture tracking for regularization. The
corresponding task weights were respectively set to $(10^4, 10^2, 10, 1,
10^{-1})$. Each QP solution provides joint-angle velocities $\qd_\textrm{ref}$,
which is then sent to the robot. See Figure~\ref{fig:pipeline} for a summary of
our pipeline. 

\begin{figure}[t]
    \centering
    \includegraphics[width=0.98\columnwidth]{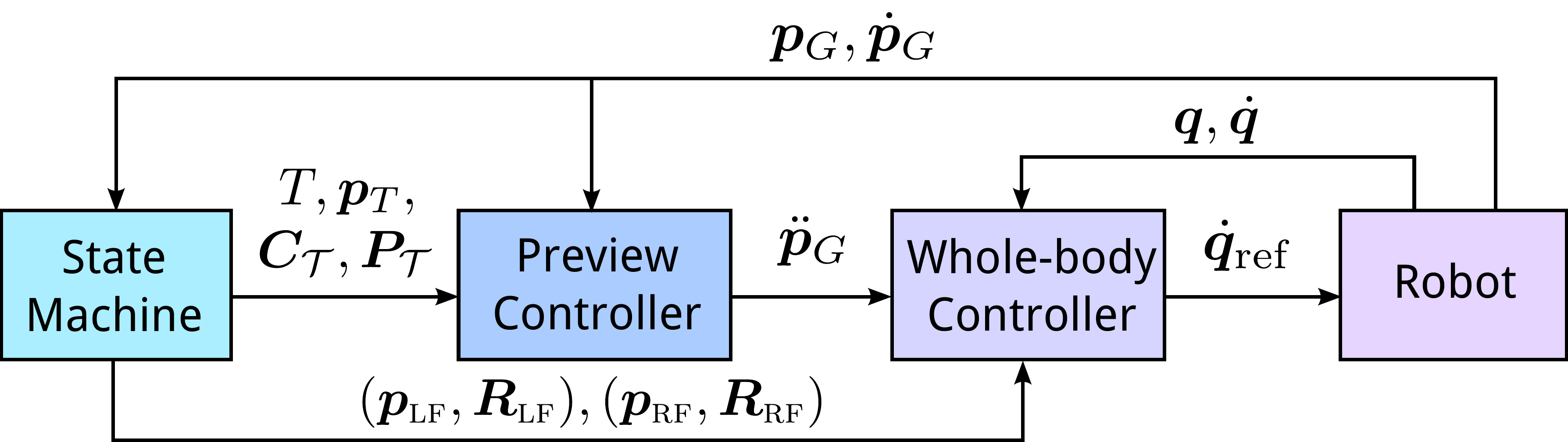}
    \caption{
        Overview of our control pipeline. The COM trajectory polyhedron
        $\bfP_\calT$ and its corresponding contact-stability cone $\bfC_\calT$
        are computed using the method described in Section~\ref{mpc101}.
    }
    \label{fig:pipeline}
\end{figure}

\section{Experiments}
\label{sec:xp}

The ground truth for contact stability is the existence of feasible contact
forces $\bff_\all$ summing up to the net wrench $\hat{\bfw}$ of the motion. In
all experiments, we validate our trajectories by checking the existence of such
$\bff_\all$ at each time instant. In both experiments, friction coefficients
were set to $\mu=0.7$.

\subsection{Regular staircase and walking into an aircraft}

Our first scenario, provided by Airbus Group, takes place in a 3D realistic
model of scale 1:1 for a section of the A350 airplane. This mock-up will be
used for experiments with the real robot when the research matures to an
integrated software, and gets approval from Airbus Group. In order to access a
predefined spot in the A350, the robot needs to climb stairs (size of those
available in the factory), walk on a platform (flat ground) to finally reach
the mockup floor composed of removable tiles that can be uneven and disposed in
various locations, see Figure~\ref{fig:airbus}. In this simulation, the
footprints were given together with the timing for the \textsf{\textsc{ds}} and
\textsf{\textsc{ss}} phases (respectively 0.5~s and 1~s).

\subsection{Slanted circular staircase with tilted steps}

The slanted circular staircase depicted in Figure~\ref{fig:staircase} has 26
steps randomly rolled, pitched and yawed by angles $(\theta_r, \theta_p,
\theta_y) \in [-0.5, +0.5]^3$~rad. The average radius of the staircase is $1.4$
m, and the altitude difference between the highest and lower steps is also
$1.4$~m.

For this scenario, the reference durations of single and double support phases
were set to $T_\ss = 1$~s and $T_\ds = 0.5$~s, respectively. We concur
with~\cite{audren2014iros} that the question of finding proper timings becomes
crucial in multi-contact. Having constant durations overlooks the fact that
some steps are harder to take than others (due to their respective
tilting, altitude difference, etc.).\footnote{This question is less critical
for walking on horizontal or well structured floors, where all steps are
similar.} Being unable to find a single pair of constants $(T_\ss, T_\ds)$
suited to the whole staircase, we opted for the following workaround condition:
\begin{itemize}
    \item [\textsc{w)}] At the end of a double-support phase, wait for the COM to be
        above the static-equilibrium polygon of the next single-support before
        activating the phase transition.
\end{itemize}
This choice is motivated by the link~\eqref{link-with-sep} between the 3D cone
of COM accelerations and the position of the COM in the static-equilibrium
polygon. In practice, it allows the use of ``optimistic'' values of $(T_\ss,
T_\ds)$ while only extending $T_\ds$ when necessary. 

\begin{figure}[!t]
    \centering
    \includegraphics[width=0.98\columnwidth]{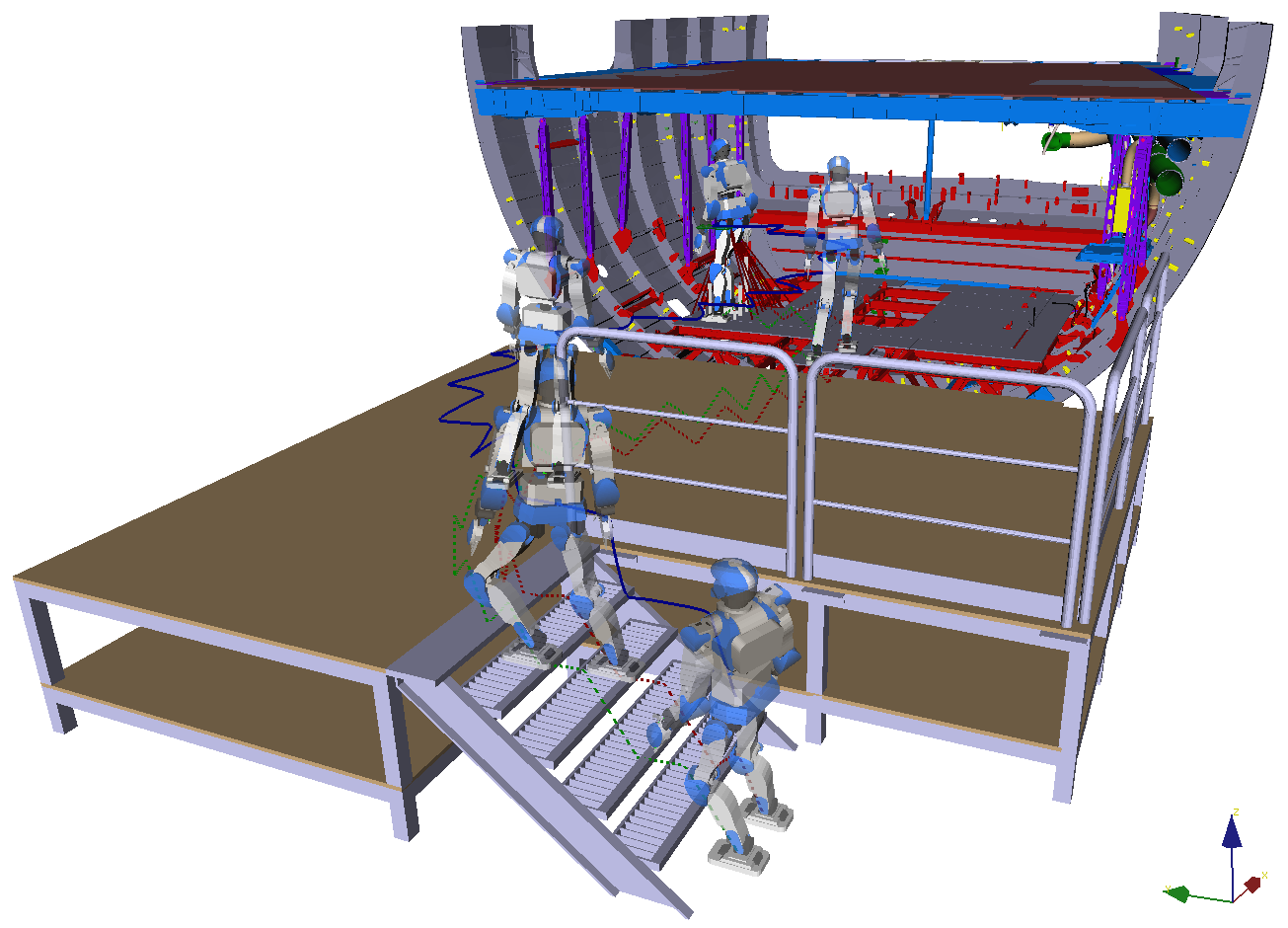}
    \caption{
        HRP-4 accessing the floor-shop of an A350 through stairs and reach the
        working areas by walking on tiles. Foot trajectories (dashed lines)
        step over two staircase steps at a time, as done in natural stair
        climbing. The COM trajectory (blue line) illustrates the progression of
        the robot until its target configuration inside the aircraft. Contact
        stability of the whole motion was cross-validated by checking the
        existence of groundtruth contact forces $\bff_\all$ at each timestep,
        as shown in the accompanying video~\cite{code}.
    }
    \label{fig:airbus}
\end{figure}

\section{Conclusion}

We presented a multi-contact walking pattern generator based on preview-control
of the 3D acceleration of the center of mass. Our development builds upon
algebraic manipulations of friction cones as dual twists, thanks to which we
can recompute 3D cones of feasible COM accelerations in real-time. We then
showed how to intersect these cones over the preview window of a model-preview
controller to construct a conservative trajectory-wide contact-stability
criterion. We implemented this pipeline and illustrated it with the HRP-4
humanoid model in multi-contact dynamically walking scenarios. All our source
code is released at~\cite{code}.

\section*{Acknowledgment}

The authors would like to warmly thank Komei Fukuda for his helpful feedback
and Patrick Wensing for pointing out a calculation mistake in a preliminary
version of the paper. This paper has also benefitted from rich discussions with
Herv\'e Audren and Adrien Escande.

\bibliographystyle{IEEEtran}
\bibliography{refs}

% Generated by IEEEtran.bst, version: 1.13 (2008/09/30)
\begin{thebibliography}{10}
\providecommand{\url}[1]{#1}
\csname url@samestyle\endcsname
\providecommand{\newblock}{\relax}
\providecommand{\bibinfo}[2]{#2}
\providecommand{\BIBentrySTDinterwordspacing}{\spaceskip=0pt\relax}
\providecommand{\BIBentryALTinterwordstretchfactor}{4}
\providecommand{\BIBentryALTinterwordspacing}{\spaceskip=\fontdimen2\font plus
\BIBentryALTinterwordstretchfactor\fontdimen3\font minus
  \fontdimen4\font\relax}
\providecommand{\BIBforeignlanguage}[2]{{%
\expandafter\ifx\csname l@#1\endcsname\relax
\typeout{** WARNING: IEEEtran.bst: No hyphenation pattern has been}%
\typeout{** loaded for the language `#1'. Using the pattern for}%
\typeout{** the default language instead.}%
\else
\language=\csname l@#1\endcsname
\fi
#2}}
\providecommand{\BIBdecl}{\relax}
\BIBdecl

\bibitem{code}
\BIBentryALTinterwordspacing
 [Online]. Available: \url{https://github.com/stephane-caron/3d-mpc}
\BIBentrySTDinterwordspacing

\bibitem{englsberger2015tro}
J.~Englsberger, C.~Ott, and A.~Albu-Schaffer, ``Three-dimensional bipedal
  walking control based on divergent component of motion,'' \emph{IEEE
  Transactions on Robotics}, vol.~31, no.~2, pp. 355--368, 2015.

\bibitem{zhao2016rss}
Y.~Zhao, B.~Fernandez, and L.~Sentis, ``Robust phase-space planning for agile
  legged locomotion over various terrain topologies,'' in \emph{Robotics:
  Science and System}, 2016.

\bibitem{hirukawa2007icra}
H.~Hirukawa, S.~Hattori, S.~Kajita, K.~Harada, K.~Kaneko, F.~Kanehiro,
  M.~Morisawa, and S.~Nakaoka, ``A pattern generator of humanoid robots walking
  on a rough terrain,'' in \emph{IEEE International Conference on Robotics and
  Automation}.\hskip 1em plus 0.5em minus 0.4em\relax IEEE, 2007, pp.
  2181--2187.

\bibitem{audren2014iros}
H.~Audren, J.~Vaillant, A.~Kheddar, A.~Escande, K.~Kaneko, and E.~Yoshida,
  ``Model preview control in multi-contact motion-application to a humanoid
  robot,'' in \emph{IEEE/RSJ International Conference on Intelligent Robots and
  Systems}.\hskip 1em plus 0.5em minus 0.4em\relax IEEE, 2014, pp. 4030--4035.

\bibitem{herzog2015humanoids}
A.~Herzog, N.~Rotella, S.~Schaal, and L.~Righetti, ``Trajectory generation for
  multi-contact momentum control,'' in \emph{Humanoid Robots (Humanoids), 2015
  IEEE-RAS 15th International Conference on}.\hskip 1em plus 0.5em minus
  0.4em\relax IEEE, 2015, pp. 874--880.

\bibitem{carpentier2015icra}
J.~Carpentier, S.~Tonneau, M.~Naveau, O.~Stasse, and N.~Mansard, ``A versatile
  and efficient pattern generator for generalized legged locomotion,'' in
  \emph{2016 IEEE International Conference on Robotics and Automation (ICRA)},
  May 2016, pp. 3555--3561.

\bibitem{brasseur2015humanoids}
C.~Brasseur, A.~Sherikov, C.~Collette, D.~Dimitrov, and P.-B. Wieber, ``A
  robust linear mpc approach to online generation of 3d biped walking motion,''
  in \emph{Humanoid Robots (Humanoids), 2015 IEEE-RAS 15th International
  Conference on}.\hskip 1em plus 0.5em minus 0.4em\relax IEEE, 2015, pp.
  595--601.

\bibitem{naveau2017ral}
M.~Naveau, M.~Kudruss, O.~Stasse, C.~Kirches, K.~Mombaur, and P.~Sou{\`e}res,
  ``A reactive walking pattern generator based on nonlinear model predictive
  control,'' \emph{IEEE Robotics and Automation Letters}, vol.~2, no.~1, pp.
  10--17, 2017.

\bibitem{vanheerden2017ral}
K.~{Van Heerden}, ``Real-time variable center of mass height trajectory
  planning for humanoids robots,'' \emph{IEEE Robotics and Automation Letters},
  vol.~2, no.~1, pp. 135--142, 2017.

\bibitem{caron2016tro}
S.~Caron, Q.-C. Pham, and Y.~Nakamura, ``{ZMP} support areas for multi-contact
  mobility under frictional constraints,'' \emph{IEEE Transactions on
  Robotics}, to appear.

\bibitem{featherstone2014}
R.~Featherstone, \emph{Rigid body dynamics algorithms}.\hskip 1em plus 0.5em
  minus 0.4em\relax Springer, 2014.

\bibitem{caron2015icra}
S.~Caron, Q.-C. Pham, and Y.~Nakamura, ``Stability of surface contacts for
  humanoid robots: Closed-form formulae of the contact wrench cone for
  rectangular support areas,'' in \emph{IEEE International Conference on
  Robotics and Automation}.\hskip 1em plus 0.5em minus 0.4em\relax IEEE, 2015.

\bibitem{kajita2003icra}
S.~Kajita, F.~Kanehiro, K.~Kaneko, K.~Fujiwara, K.~Harada, K.~Yokoi, and
  H.~Hirukawa, ``Biped walking pattern generation by using preview control of
  zero-moment point,'' in \emph{IEEE International Conference on Robotics and
  Automation}, vol.~2.\hskip 1em plus 0.5em minus 0.4em\relax IEEE, 2003, pp.
  1620--1626.

\bibitem{mordatch2010tog}
I.~Mordatch, M.~De~Lasa, and A.~Hertzmann, ``Robust physics-based locomotion
  using low-dimensional planning,'' \emph{ACM Transactions on Graphics
  (SIGGRAPH)}, vol.~29, no.~4, p.~71, July 2010.

\bibitem{qiu2011isdhm}
Z.~Qiu, A.~Escande, A.~Micaelli, and T.~Robert, ``Human motions analysis and
  simulation based on a general criterion of stability,'' in
  \emph{International Symposium on Digital Human Modeling}, 2011.

\bibitem{caron2015rss}
S.~Caron, Q.-C. Pham, and Y.~Nakamura, ``Leveraging cone double description for
  multi-contact stability of humanoids with applications to statics and
  dynamics,'' in \emph{Robotics: Science and System}, 2015.

\bibitem{bretl2008tro}
T.~Bretl and S.~Lall, ``Testing static equilibrium for legged robots,''
  \emph{IEEE Transactions on Robotics}, vol.~24, no.~4, pp. 794--807, 2008.

\bibitem{delprete2016icra}
A.~Del~Prete, S.~Tonneau, and N.~Mansard, ``Fast algorithms to test robust
  static equilibrium for legged robots,'' in \emph{IEEE International
  Conference on Robotics and Automation}, Stockholm, Sweden, May 2016.

\bibitem{zhang2016ijhr}
T.~Zhang, S.~Caron, and Y.~Nakamura, ``Humanoid stair climbing based on
  dedicated plane segment estimation and multi-contact motion generation,''
  \emph{to appear in: International Journal of Humanoid Robotics}, 2016.

\bibitem{fukuda1996double}
K.~Fukuda and A.~Prodon, ``Double description method revisited,'' in
  \emph{Combinatorics and computer science}.\hskip 1em plus 0.5em minus
  0.4em\relax Springer, 1996, pp. 91--111.

\bibitem{bouyarmane2009icra}
K.~Bouyarmane, A.~Escande, F.~Lamiraux, and A.~Kheddar, ``Potential field guide
  for humanoid multicontacts acyclic motion planning,'' in \emph{IEEE
  International Conference on Robotics and Automation}.\hskip 1em plus 0.5em
  minus 0.4em\relax IEEE, 2009, pp. 1165--1170.

\bibitem{escande2013ras}
A.~Escande, A.~Kheddar, and S.~Miossec, ``Planning contact points for humanoid
  robots,'' \emph{Robotics and Autonomous Systems}, vol.~61, no.~5, pp.
  428--442, 2013.

\bibitem{avis1992dcg}
D.~Avis and K.~Fukuda, ``A pivoting algorithm for convex hulls and vertex
  enumeration of arrangements and polyhedra,'' \emph{Discrete {\&}
  Computational Geometry}, vol.~8, no.~3, pp. 295--313, 1992.

\bibitem{kirkpatrick1986siam}
D.~G. Kirkpatrick and R.~Seidel, ``The ultimate planar convex hull algorithm?''
  \emph{SIAM Journal on Computing}, vol.~15, no.~1, pp. 287--299, 1986.

\bibitem{boyd2004convex}
S.~Boyd and L.~Vandenberghe, \emph{Convex Optimization}, 2004.

\bibitem{barber1996quickhull}
C.~B. Barber, D.~P. Dobkin, and H.~Huhdanpaa, ``The quickhull algorithm for
  convex hulls,'' \emph{ACM Transactions on Mathematical Software}, vol.~22,
  no.~4, pp. 469--483, 1996.

\bibitem{pham2015tm}
Q.-C. Pham and O.~Stasse, ``Time-optimal path parameterization for
  redundantly-actuated robots: A numerical integration approach,''
  \emph{IEEE/ASME Transactions on Mechatronics}, 2015.

\end{thebibliography}

\appendix

We compare the convex hull reduction to the original
calculation~\cite{bretl2008tro} of the static-equilibrium polygon. Computation
times for randomly sampled contact configurations are reported in
Table~\ref{table:times2} for four algorithms:
\begin{itemize}
    \item \emph{cdd + hull}: the method described in
        Section~\ref{vertex-enum}, where \emph{cdd}~\cite{fukuda1996double}
        is used to compute the CWC while convex hulls are computed with
        \emph{Qhull}~\cite{barber1996quickhull}.
    \item \emph{Parma + hull}: same approach, using the Parma
        Polyhedra Library\footnote{\url{https://github.com/haudren/pyparma}}
        rather than \emph{cdd} to compute the CWC. 
    \item \emph{cdd only}: as described in~\cite{zhang2016ijhr}, \emph{cdd} can
        also be used to compute the static-equilibrium polygon directly.
    \item \emph{Bretl \& Lall}: the algorithm from~\cite{bretl2008tro}, in the
        implementation from~\cite{pham2015tm} but using GLPK as LP
        solver.\footnote{Using an efficient LP solver is crucial here: in a
        preliminary version of this paper, we used the more general CVXOPT,
        which resulted in computations around $10 \times$ slower than those we
        now report using GLPK.}
\end{itemize}
The \emph{Parma + hull} solution is the slowest but most numerically stable,
while \emph{cdd only} is only competitive in single-support. Neck to neck are
\emph{Bretl \& Lall} and \emph{cdd + hull}, with the latter faster in single-
and double-support. But the real benefit of our approach comes with the
computation of time-varying criteria: in double- and triple-support, we see
that executing \emph{hull only} is more than ten times faster than applying any
other algorithm from scratch. We also highlight it as the fastest solution for
single-support, as in this case the CWC is known
analytically~\cite{caron2015icra} and there is no need for the \emph{cdd} step.

\begin{table}[h]
    \caption{
        Time (in ms) to compute the static-equilibrium polygon, averaged over
        100 random contact configurations.
    }
    \label{table:times2}
    \begin{tabular}{rrrr}
        Algorithm & Single support & Double support & Triple support \\
        \hline
        Parma + hull                     & $6.02 \pm 0.20$  & $21.0 \pm 4.2$ & $42 \pm 11$ \\
        cdd only                         & $0.38 \pm 0.01$  & $7.0  \pm 2.7$ & $> 500$ \\
        Brel \& Lall~\cite{bretl2008tro} & $1.00 \pm 0.02$  & $3.1  \pm 0.8$ & $\bf 5.9 \pm 1.6$ \\
        cdd + hull                       & $0.60 \pm 0.01$  & $\bf 2.7  \pm 0.6$ & $7.1 \pm 1.9$ \\
        \hline
        \it hull only                    & $\bf 0.17 \pm 0.003$ & $\it 0.28 \pm
        0.04$ & $\it 0.38 \pm 0.09$ \\
        \hline
    \end{tabular}
\end{table}

\end{document}